%% file: main.tex
\documentclass[10pt,letterpaper]{article}
\usepackage[letterpaper,textwidth=5.5in,textheight=9in]{geometry}
\usepackage{times,natbib}
\usepackage{amsmath,amssymb,mathtools,bm,amsthm}
\newtheorem{lemma}{Lemma}
\usepackage{microtype}
\usepackage{graphicx}
\usepackage[caption=false]{subfig}
\input{tables/results_macros}

\input{tables/bootstrap_macros}
\input{tables/adaptive_bootstrap_macros}

\input{tables/execution_score_macros}

\usepackage{booktabs,tabularx,array}
\usepackage{float}
\floatstyle{ruled}
\newfloat{algorithm}{tbp}{loa}
\floatname{algorithm}{Algorithm}
\usepackage{xcolor}
\usepackage[colorlinks=true,linkcolor=blue!40!black,citecolor=blue!40!black,urlcolor=blue!40!black]{hyperref}
\hypersetup{pdftitle={Role-Adaptive Policy Optimization for Offline Reinforcement Learning},pdfauthor={Seonvin Cho, Soohyun Choi, Songnam Hong}}
\usepackage{enumitem}
\usepackage{needspace}
\newcommand{\E}{\mathbb{E}}
\newcommand{\D}{\mathcal{D}}
\newcommand{\Bin}{\mathcal{B}_{\mathrm{in}}}
\newcommand{\Bout}{\mathcal{B}_{\mathrm{out}}}
\newcommand{\sg}{\operatorname{sg}}

\newcommand{\softplus}{\operatorname{softplus}}
\title{Role-Adaptive Policy Optimization\\for Offline Reinforcement Learning}
\author{Seonvin Cho \qquad Soohyun Choi \qquad Songnam Hong\thanks{Corresponding author: \texttt{snhong@hanyang.ac.kr}.}\\[4pt]
Department of Electronic Engineering, Hanyang University\\
Seoul, Republic of Korea\\[3pt]
\texttt{seonbin0319@hanyang.ac.kr} \quad \texttt{petersun0221@hanyang.ac.kr}\\
\texttt{snhong@hanyang.ac.kr}}
\date{September 2026}

\begin{document}
\maketitle

\begin{abstract}
Policy regularization in offline reinforcement learning balances policy improvement against reliance on uncertain value estimates. This balance can differ between selecting actions for execution and supplying actions for critic bootstrapping, yet methods such as TD3+BC couple these roles through a shared policy. We propose Role-Adaptive Policy Optimization (RAPO), which adapts policy-update coefficients according to their roles in value learning and execution. RAPO learns these coefficients by differentiating through candidate policy updates formed using the base algorithm's actor objective. For TD3+BC, RAPO separates bootstrap and execution actors and adapts their coefficients independently: the bootstrap objective penalizes policy-induced changes in target values, while the execution objective evaluates a local policy-improvement surrogate. For IQL, whose value learning is already independent of the execution actor, RAPO preserves the original value updates and adapts only the inverse temperature in advantage-weighted policy extraction. Experiments on D4RL locomotion and AntMaze tasks show improvements over both base algorithms, with larger gains for TD3+BC, whose RAPO instantiation outperforms baselines on average.
\end{abstract}

\section{Introduction}
\label{sec:introduction}

Many offline reinforcement learning (RL) methods improve a policy by balancing critic-guided updates against a reference derived from recorded behavior \citep{fujimoto2021minimalist,tarasov2023rebrac}. The critic is learned from a fixed dataset, so its estimates can be unreliable for actions poorly covered by that dataset \citep{levine2020offline}. The regularization coefficient controls how strongly an update follows these estimates. A setting that is useful at one stage of training may become overly restrictive or permissive as the actor and critic evolve.

The appropriate balance also depends on the role of the policy. In TD3+BC, one actor selects actions for execution and, through its target copy, supplies next-state actions for critic bootstrapping. An update toward higher estimated values can improve the execution policy while changing the critic's regression targets. Conversely, limiting target movement can restrict useful execution updates. These two uses therefore need not favor the same regularization, even when they share a critic.

IQL learns values without using its extracted actor, and methods with separate target and evaluation policies allow different constraint strengths \citep{kostrikov2022iql,xu2024mcep}. Policy extraction can substantially affect performance for a given value function \citep{hansenestruch2023idql,park2024bottleneck}. Adaptive-constraint methods can learn regularization coefficients through differentiable actor updates \citep{jing2026aspc}. The evaluation criterion for such an update should reflect whether the policy supplies execution actions or critic targets.

We propose \emph{Role-Adaptive Policy Optimization} (RAPO), which learns regularization coefficients by evaluating candidate policy updates according to how their actions are used. Execution updates are assessed for local policy improvement, while bootstrap updates are assessed for the stability of the critic targets they produce. RAPO differentiates through a candidate update to choose its coefficient and retains the base algorithm's actor objective for the actual update.

We instantiate this procedure in TD3+BC and IQL. TD3+RAPO separates execution and bootstrap actors and learns their coefficients independently. IQL learns Q and V from dataset actions without using its execution actor in either update \citep{kostrikov2022iql}. IQL+RAPO therefore adapts only the inverse temperature of advantage-weighted policy extraction. An explicit bootstrap actor is used when the base method obtains critic targets from policy actions.

The contribution is a role-specific criterion for learning each policy-update coefficient while retaining the base actor losses. Across fifteen D4RL tasks, both instantiations improve the domain averages of their base algorithms, and TD3+RAPO has the highest overall average among the compared methods. Matched TD3 controls examine the execution score and bootstrap adaptation; a fixed-coefficient control tests whether learned endpoint values reproduce adaptation during training.

\section{Background}
\label{sec:background}

\subsection{Offline Learning Setup}

The training experience in offline RL is a fixed dataset $\D=\{(s,a,r,s')\}$ collected by behavior policies \citep{levine2020offline}. Each tuple records a state $s\in\mathcal S$, action $a\in\mathcal A$, reward $r$, and successor state $s'$. Learning uses these records without further environment interaction.

We seek a policy $\pi$ that maximizes $J(\pi)=\E_{\rho_0,P,\pi}[\sum_{t=0}^{\infty}\gamma^t r(s_t,a_t)]$, where $\rho_0$ is the initial-state distribution, $P$ is the environment transition kernel, $r$ is the reward function, and $\gamma\in[0,1)$ is the discount factor. The actor-critic methods below represent the policy by $\pi_\theta$ and estimate action and state values with $Q_\phi$ and $V_\psi$, respectively.

\subsection{Base Actor-Critic Updates}
\label{sec:baseupdates}

TD3+BC and IQL differ in their actor objectives and in the actions used for critic bootstrapping. Let $\bar Q_j$ denote target critic $j$ and $\bar Q_{\min}=\min_{j=1,2}\bar Q_j$ their pointwise minimum. Bars consistently mark target critics. Let $\sg$ denote the stop-gradient operator. We write the losses for nonterminal transitions; implementation details are given in Appendix~\ref{app:training}.

\paragraph{TD3+BC.}
TD3+BC combines Q maximization with behavior cloning \citep{fujimoto2021minimalist}. In the loss convention used by RAPO, its actor minimizes
\begin{equation}
\label{eq:td3inner}
\mathcal L_\pi^{\mathrm{TD3+BC}}(\theta;\alpha)
=\E_{(s,a)\sim\D}\left[
-\frac{Q_{\phi_1}(s,\textcolor{red!65!black}{\pi_\theta(s)})}{S}
+\frac{\|\textcolor{red!65!black}{\pi_\theta(s)}-a\|_2^2}{\alpha d_a}
\right],
\end{equation}
where $\alpha>0$, $d_a$ is the action dimension, and
$S=\sg(\max\{\E_{s\sim\D}|Q_{\phi_1}(s,\pi_\theta(s))|,\epsilon\})$
for a small $\epsilon>0$. Increasing $\alpha$ weakens behavior regularization relative to Q maximization. The critic loss is
\begin{equation}
\label{eq:td3basetarget}
\mathcal L_Q^{\mathrm{TD3}}(\phi)
=\sum_{j=1}^{2}\E_{(s,a,r,s')\sim\D}
\left[\left(Q_{\phi_j}(s,a)
-\sg\!\left(r+\gamma\bar Q_{\min}(s',\textcolor{blue!65!black}{\bar\pi_\theta(s')})\right)\right)^2\right].
\end{equation}
In TD3+BC, the actor in Equation~\eqref{eq:td3inner} selects execution actions; its target copy in Equation~\eqref{eq:td3basetarget} selects the next-state actions used for critic bootstrapping. RAPO assigns these two uses to separate actors.

\paragraph{IQL.}
IQL learns values from dataset actions independently of its actor \citep{kostrikov2022iql}:
\begin{align}
\label{eq:iqlvalue}
\mathcal L_V(\psi)
&=\E_{(s,a)\sim\D}\ell_\tau\!\left(\sg(\bar Q_{\min}(s,a))-V_\psi(s)\right),\\
\label{eq:iqltarget}
\mathcal L_Q^{\mathrm{IQL}}(\phi)
&=\frac12\sum_{j=1}^{2}\E_{(s,a,r,s')\sim\D}
\left[\left(Q_{\phi_j}(s,a)-\sg\!\left(r+\gamma V_\psi(s')\right)\right)^2\right],
\end{align}
where $\ell_\tau(u)=|\tau-\mathbf1\{u<0\}|u^2$ and $\tau\in[1/2,1)$ select an upper expectile (the mean at $\tau=1/2$). Given the detached advantage $A(s,a)=\sg(\bar Q_{\min}(s,a)-V_\psi(s))$, the actor minimizes
\begin{equation}
\label{eq:iqlactor}
\mathcal L_\pi^{\mathrm{IQL}}(\theta;\beta)
=-\E_{(s,a)\sim\D}\left[\exp(\beta A(s,a))\log\textcolor{red!65!black}{\pi_\theta(a\mid s)}\right].
\end{equation}
The inverse temperature $\beta>0$ controls how strongly policy extraction favors higher-advantage actions.

\subsection{Meta-Gradient Learning}
\label{sec:metabackground}

Meta-gradient learning trains a coefficient by differentiating through the update it controls \citep{xu2018metagradient}. For an actor loss $\mathcal L_{\mathrm{inner}}$ and dataset minibatch $\Bin$, an \emph{inner update} changes $\theta$ with $\alpha$ fixed. A gradient step with learning rate $\eta$ gives
\[
\widetilde\theta(\alpha)
=\theta-\eta\nabla_\theta\mathcal L_{\mathrm{inner}}(\theta;\alpha,\Bin).
\]
An \emph{outer loss} $\mathcal L_{\mathrm{outer}}$ scores the updated actor on another minibatch $\Bout$. Holding the starting parameters $\theta$ fixed, the chain rule gives
\begin{equation}
\label{eq:meta}
\frac{\mathrm d\mathcal L_{\mathrm{outer}}(\widetilde\theta(\alpha);\Bout)}{\mathrm d\alpha}
=\left(\frac{\partial\widetilde\theta}{\partial\alpha}\right)^{\!\top}
\nabla_{\widetilde\theta}\mathcal L_{\mathrm{outer}}.
\end{equation}
Gradient descent on the outer loss trains $\alpha$: the inner loss trains the policy, and the outer loss trains its update coefficient.

\section{Method}
\label{sec:method}

RAPO applies the meta-gradient procedure in Section~\ref{sec:metabackground} with outer losses tailored to execution and critic bootstrapping (Figure~\ref{fig:framework}).

\subsection{Policy Roles and Adaptive Coefficients}
\label{sec:roles}

The execution policy $\pi_E$ selects actions in the environment. When value learning requires a policy to supply next-state actions, RAPO assigns that task to a separate bootstrap policy $\pi_B$. For each active role $i$, write $c_i$ for its method-specific positive actor-loss coefficient:
\begin{center}
\begin{tabular}{lccc}
\toprule
Action use & Actor & Coefficient & Outer criterion \\
\midrule
Execution & $\pi_E=\pi_{\theta_E}$ & $c_E$ & Local value improvement \\
Critic-target actions & $\pi_B=\pi_{\theta_B}$ & $c_B$ & Target stability \\
\bottomrule
\end{tabular}
\end{center}
TD3+RAPO uses both actors with shared critics, with $c_E=\alpha_E$ and $c_B=\alpha_B$: $\pi_E$ is executed, while $\pi_B$ supplies the actions in Equation~\eqref{eq:td3basetarget}. IQL's target in Equation~\eqref{eq:iqltarget} is actor-independent, so its only required role is execution, with $c_E=\beta_E$.

For each active role $i$, the inner loss $\mathcal L_{\mathrm{inner},i}$ is the corresponding actor objective from Section~\ref{sec:baseupdates}: Equation~\eqref{eq:td3inner} with $\alpha=c_i$ for TD3+BC, or Equation~\eqref{eq:iqlactor} with $\beta=c_E$ for IQL.

Let $\mathcal U_i(\theta_i;c_i,\Bin)$ denote one differentiable optimizer step on $\mathcal L_{\mathrm{inner},i}$ from $\theta_i$, using coefficient $c_i$ and minibatch $\Bin$. The resulting candidate parameters and policy are
\begin{equation}
\label{eq:virtual}
\widetilde\theta_i(c_i)
=\mathcal U_i(\theta_i;c_i,\Bin),
\qquad \widetilde\pi_i=\pi_{\widetilde\theta_i(c_i)}.
\end{equation}
The outer losses evaluate $\widetilde\pi_i$ on an independently drawn minibatch $\Bout$ and update $c_i$ through Equation~\eqref{eq:meta}. Positive, detached scales $\bar S_i$ normalize the scores by mean absolute critic values at candidate actions. Appendix~\ref{app:differentiation} specifies the scales and stop-gradient operations; Appendix~\ref{app:training} gives the optimizers and schedules.

\begin{figure}[H]
\centering
\subfloat[TD3+RAPO\label{fig:framework-td3}]{%
  \includegraphics[height=0.285\textwidth]{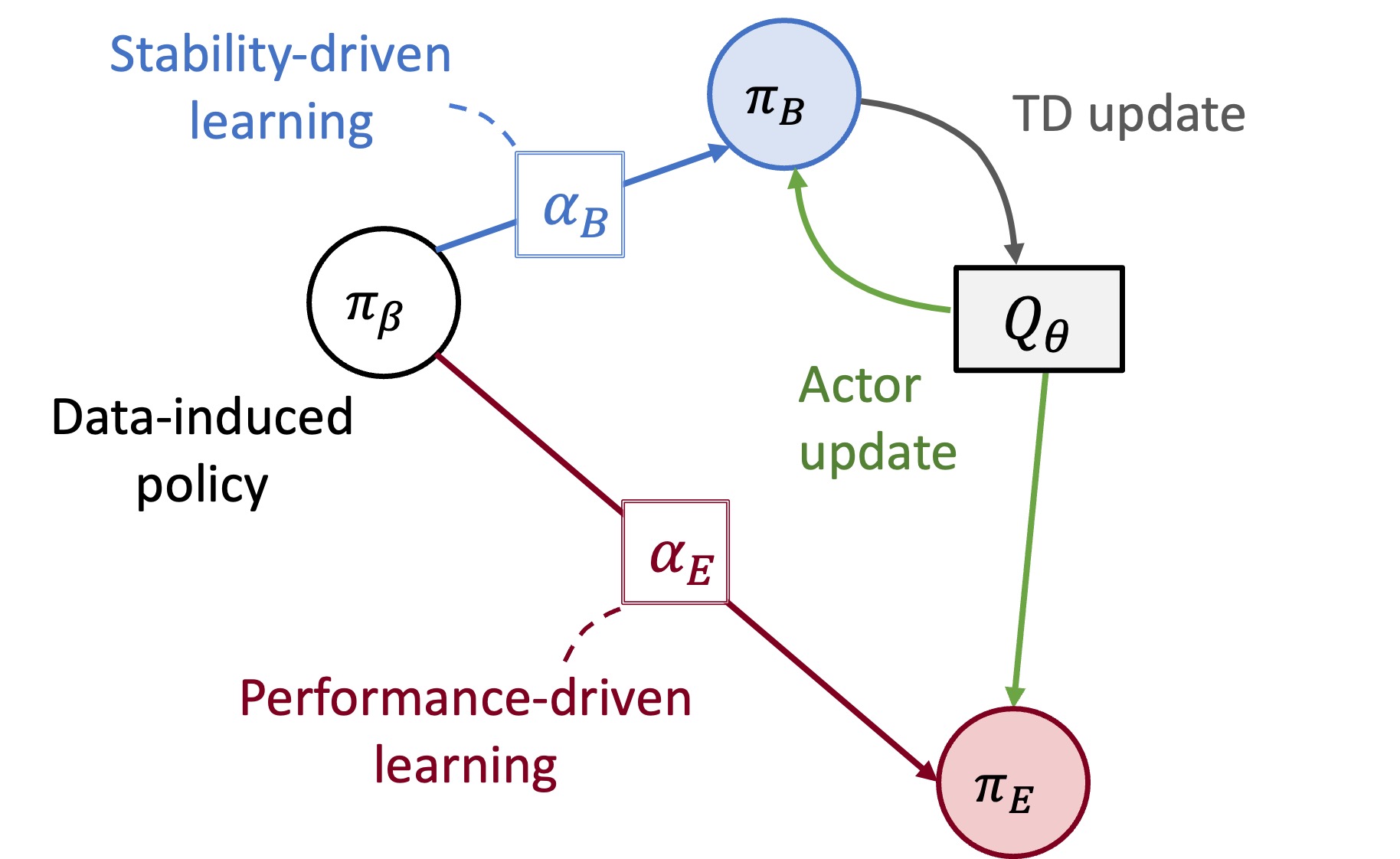}%
}\hfill
\subfloat[IQL+RAPO\label{fig:framework-iql}]{%
  \includegraphics[height=0.285\textwidth]{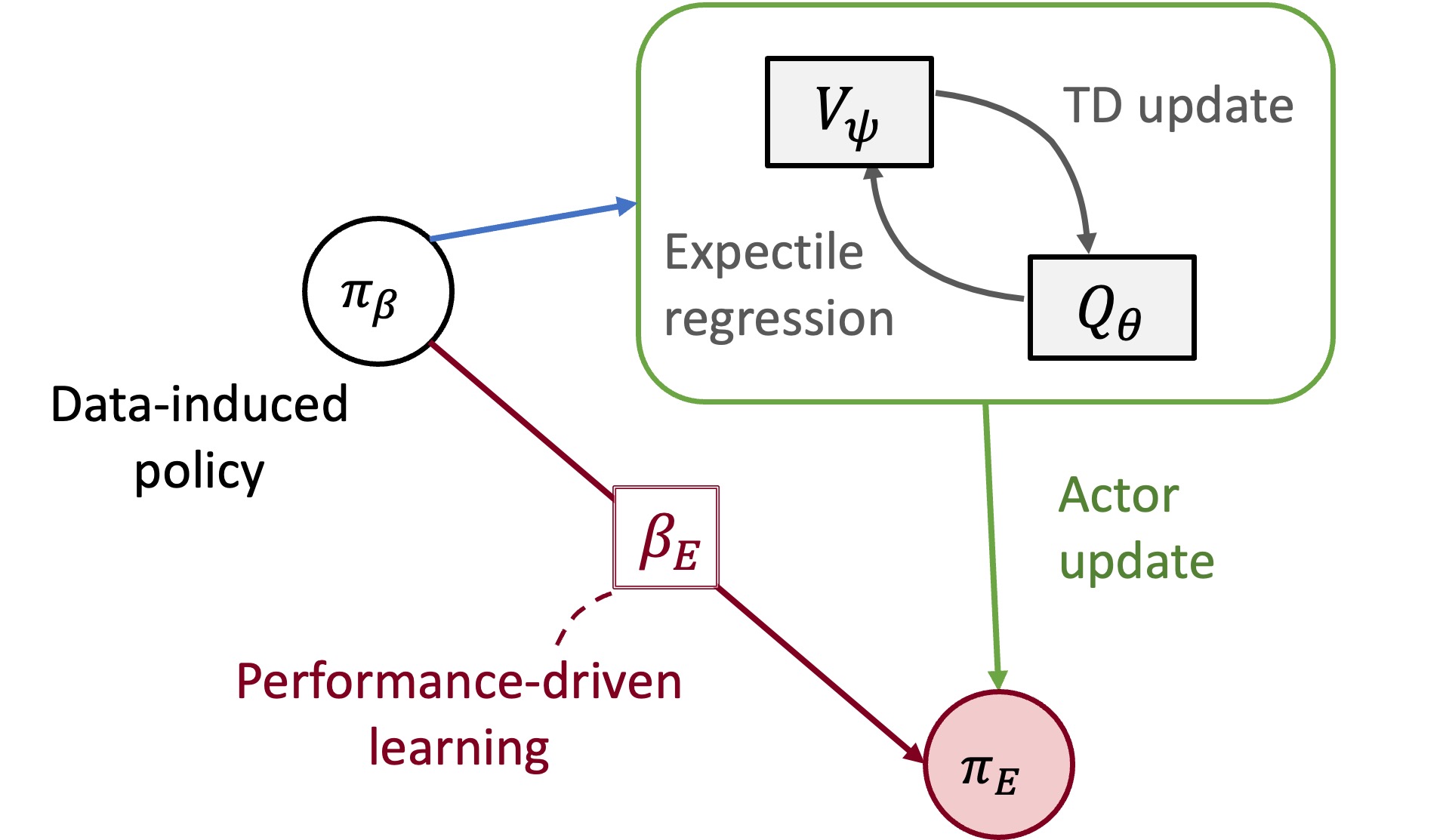}%
}
\caption{\textbf{Overview of RAPO.} $\pi_\beta$ denotes the behavior represented by the offline data. (a) TD3+RAPO separates the bootstrap actor $\pi_B$ from the execution actor $\pi_E$ and adapts their coefficients $\alpha_B$ and $\alpha_E$ using target-stability and local-improvement objectives, respectively. (b) IQL+RAPO retains actor-independent Q/V learning from dataset transitions and adapts only the execution policy's inverse temperature $\beta_E$.}
\label{fig:framework}
\end{figure}

\subsection{Execution Objective: Local Policy Improvement}
\label{sec:execution}

The performance-difference identity relates return improvement to the candidate policy's action-value gain under $Q^{\pi_E}$, averaged over its discounted state distribution (Appendix~\ref{app:return}). RAPO approximates this criterion using a learned critic and dataset states, and scores the local change from the current action to the candidate action.

Write $\bar Q_E$ for the target critic used to score execution: $\bar Q_E=\bar Q_1$ in TD3+RAPO and $\bar Q_E=\bar Q_{\min}$ in IQL+RAPO. It is held fixed while scoring an update. Here $\pi_E$ denotes the action map used in evaluation, including the mean for Gaussian actors.

For a state $s$, write $a_0=\pi_E(s)$, $a_+=\widetilde\pi_E(s)$, and $\delta a=a_+-a_0$. A standard smoothness bound motivates an update term that rewards first-order improvement and a conservatism term that accounts for curvature.

\begin{lemma}[Local critic improvement]
\label{lem:local}
Fix $s$. Suppose $\bar Q_E(s,\cdot)$ is differentiable on a neighborhood of the segment from $a_0$ to $a_+$ and its action gradient is $L_s$-Lipschitz on that segment. Then
\begin{equation}
\label{eq:smoothbound}
\bar Q_E(s,a_+)-\bar Q_E(s,a_0)
\geq \nabla_a\bar Q_E(s,a_0)^\top\delta a-\frac{L_s}{2}\|\delta a\|_2^2.
\end{equation}
\end{lemma}
To make the curvature penalty computable, define the detached endpoint gradients
\[
g_0=\sg(\nabla_a\bar Q_E(s,a_0)),\qquad
g_+=\sg(\nabla_a\bar Q_E(s,a_+)).
\]
Under a local quadratic approximation with Hessian $H_s$, the endpoint-gradient difference satisfies $g_+-g_0=H_s\delta a$. Hence,
\[
\left|\frac12\delta a^\top H_s\delta a\right|
\leq \frac12\|g_+-g_0\|_2\|\delta a\|_2.
\]
RAPO uses this computable penalty to discount the first-order update term, requiring two action gradients without forming a Hessian. This motivates the following surrogate score and execution loss:
\begin{align}
\label{eq:endpoint}
\widehat B_\pi(s)
&=\underbrace{g_0^\top\delta a}_{\text{update}}
-\underbrace{\frac12\|g_+-g_0\|_2\|\delta a\|_2}_{\text{conservatism}},\\
\label{eq:outere}
\mathcal L_E
&=-\frac{\E_{s\sim\Bout}\widehat B_\pi(s)}{\bar S_E}.
\end{align}
Minimizing $\mathcal L_E$ favors coefficients whose candidate updates have large first-order gains and small endpoint-gradient penalties. With $g_0$, $g_+$, and $\bar S_E$ detached, the coefficient gradient passes through $\delta a$. The penalty bounds the magnitude of a quadratic correction under the stated approximation; endpoint gradients alone do not establish a bound for a general critic. Appendix~\ref{app:return} identifies the additional critic and state-distribution errors involved in relating this score to return improvement.

\subsection{Bootstrap Objective: Target Stability}
\label{sec:bootstrap}

A bootstrap-policy update can move the critic's regression targets by changing the next-state actions in Equation~\eqref{eq:td3basetarget}. To isolate the effect of the policy update, we evaluate the current and candidate actions using the same fixed target critics.

For a transition $x=(s,a,r,s',d)$, let $d=1$ indicate termination and define
\[
y(\pi;x)=r+\gamma(1-d)\bar Q_{\min}(s',\pi(s')).
\]
The candidate-induced target change is
\begin{align}
\label{eq:targetchange}
\Delta y_B(x)
&=y(\widetilde\pi_B;x)-\sg\!\left(y(\pi_B;x)\right)\notag\\
&=\gamma(1-d)\left[
\bar Q_{\min}(s',\widetilde\pi_B(s'))-\sg\!\left(\bar Q_{\min}(s',\pi_B(s'))\right)
\right].
\end{align}
We penalize its root-mean-square magnitude, so positive and negative target changes cannot cancel:
\begin{equation}
\label{eq:outerb}
\mathcal L_B
=\sg(\alpha_B)\sqrt{\E_{x\sim\Bout}\!\left[
\left(\frac{\Delta y_B(x)}{\bar S_B}\right)^2\right]}.
\end{equation}

The critic loss in Equation~\eqref{eq:td3basetarget} depends on values at bootstrap actions. With the target critics fixed, $\Delta y_B(x)$ measures how an unsmoothed Bellman target for the same transition responds to a candidate update of the online bootstrap actor. Penalizing its RMS favors updates with smaller target sensitivity while the inner TD3+BC loss still trains the actor for policy improvement. The actual training target uses a smoothed target actor (Appendix~\ref{app:training}), so this objective is a proxy for target stability. Algorithm~\ref{alg:amo} combines the outer objectives with the base updates.

\begin{algorithm}[H]
\caption{RAPO (optimizer timing and schedules in Appendix~\ref{app:training})}
\label{alg:amo}
\noindent\textbf{Input:} Dataset $\D$; base method $m\in\{\mathrm{TD3+BC},\mathrm{IQL}\}$; active roles $\mathcal I(m)=\{E,B\}$ or $\{E\}$, respectively.
\begin{tabbing}
\quad\=\quad\=\quad\=\kill
Initialize actors, coefficients, and value/target networks.\\
\textbf{for} each training step \textbf{do}\\
\> Sample $\Bin$; apply the value updates of $m$ from Section~\ref{sec:baseupdates}.\\
\> \textbf{if} coefficient update is due \textbf{then}\\
\>\> Sample independent $\Bout$; form each role's candidate via \eqref{eq:virtual}.\\
\>\> Update each $c_i$ via \eqref{eq:meta}, using $\mathcal L_E$ in \eqref{eq:outere} or $\mathcal L_B$ in \eqref{eq:outerb}.\\
\> Apply the scheduled actor and target updates of $m$.\\
\textbf{end for}\\
\textbf{Output:} Trained value networks and actors $\{\pi_i:i\in\mathcal I(m)\}$; execute $\pi_E$.
\end{tabbing}
\end{algorithm}

\section{Experimental Results}
\label{sec:experiments}

\subsection{Evaluation Protocol}
\label{sec:evaluationprotocol}

We evaluate TD3+RAPO and IQL+RAPO against TD3+BC, IQL, wPC \citep{peng2023wpc}, A2PR \citep{liu2024a2pr}, and ASPC \citep{jing2026aspc} on fifteen D4RL tasks \citep{fu2020d4rl}: nine Locomotion tasks spanning HalfCheetah, Hopper, and Walker2d with medium, medium-replay, and medium-expert data, and six AntMaze tasks. The base actor-critic updates follow the CORL implementations \citep{tarasov2022corl}, with network architectures and training settings detailed in Appendix~\ref{app:settings}.

All final scores are measured at one million training steps: each of four independent training runs is evaluated over 50 episodes, and its episode mean gives one seed score. We report means and sample standard deviations across these four scores. Table~\ref{tab:benchmark} reports normalized returns for each task; domain and overall averages weight environments equally.

\input{tables/benchmark_results}

\subsection{Benchmark Performance}

TD3+RAPO averages \ResultTDThreeAMOLoco{} on Locomotion and \ResultTDThreeAMOAnt{} on AntMaze, compared with \ResultTDThreeBCLoco{} and \ResultTDThreeBCAnt{} for TD3+BC. IQL+RAPO averages \ResultIQLAMOLoco{} and \ResultIQLAMOAnt{}, compared with \ResultIQLLoco{} and \ResultIQLAnt{} for IQL. Both RAPO variants exceed their respective baselines in the two domain averages, and TD3+RAPO has the highest overall average among the compared methods.

\begin{figure}[H]
\centering
\includegraphics[width=\textwidth]{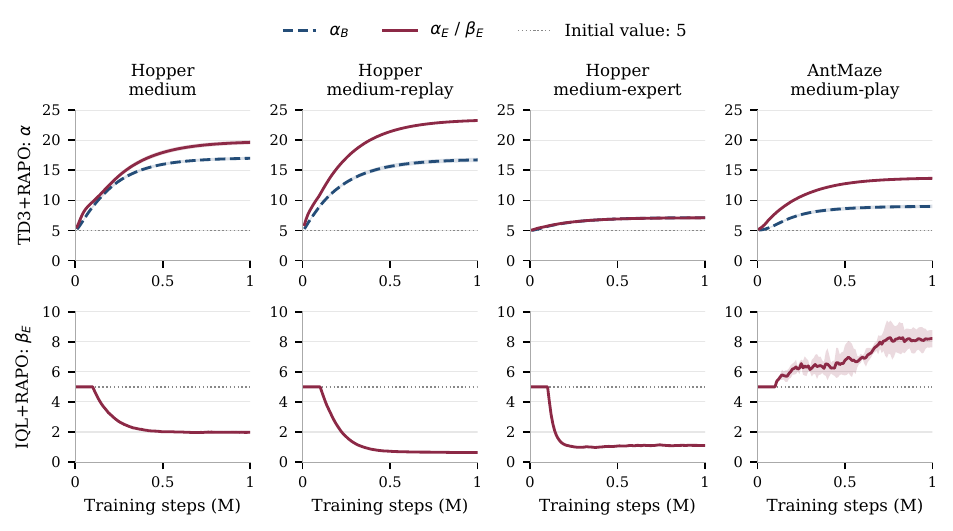}
\caption{Learned coefficients on three Hopper datasets and AntMaze-medium-play. The top row shows TD3+RAPO's $\alpha_B$ and $\alpha_E$; the bottom row shows IQL+RAPO's $\beta_E$. Curves and shaded bands show means and sample standard deviations across seeds. The dotted line marks the common initial value of five.}
\label{fig:coefficients}
\end{figure}

\Needspace{7\baselineskip}
\paragraph{Learned coefficient trajectories.}
Figure~\ref{fig:coefficients} shows the coefficients learned from equal initial values. On Hopper-medium, Hopper-medium-replay, and AntMaze-medium-play, TD3+RAPO's $\alpha_E$ rises above $\alpha_B$; on Hopper-medium-expert they remain close. In IQL+RAPO, $\beta_E$ decreases on the three Hopper datasets and increases on AntMaze-medium-play. Larger $\alpha_i$ in TD3+RAPO weakens cloning relative to Q maximization, whereas larger $\beta_E$ in IQL+RAPO increases the weight on higher-advantage dataset actions before clipping; their numerical values represent different scales.

\subsection{Execution-Objective Ablation}
\label{sec:ablations}

With the actor structure and bootstrap objective held fixed, we compare $\widehat B_\pi$ with two simpler scores to test the criterion used for execution-coefficient adaptation.

\textbf{Q} uses the direct change in the fixed execution critic defined in Section~\ref{sec:execution}, $f_Q(s)=\bar Q_E(s,a_+)-\sg(\bar Q_E(s,a_0))$, for the same candidate actions $a_0,a_+$.

\textbf{Linear} uses only the first-order gain from Equation~\eqref{eq:endpoint}, $f_{\mathrm{Linear}}(s)=g_0^\top\delta a$, omitting the endpoint-gradient penalty.

Each score defines an execution loss $-\E_{\Bout}[f(s)]/\bar S_E$, using the normalization and detachment rules in Appendix~\ref{app:differentiation}.

Q and Linear change only the execution score; they retain the two actors, RMS bootstrap objective, and coefficient learning rate of TD3+RAPO in each environment. All three conditions cover the fifteen benchmark tasks.

\input{tables/execution_score_domains}

Table~\ref{tab:execution-scores} shows that the full score $\widehat B_\pi$ in $\mathcal L_E$ has the highest means in both domains and overall. Its overall mean is \ExecutionFullOverall{}, compared with \ExecutionDirectQOverall{} for Q and \ExecutionFirstOrderOverall{} for Linear. The Linear control isolates the endpoint-gradient penalty; its lower averages suggest that this term helps select execution-coefficient updates. The full score also exceeds direct Q change under the same actor structure and bootstrap objective. These comparisons support $\mathcal L_E$ as an effective criterion for evaluating candidate execution updates in this setting.

\subsection{Adaptation During Training}
\label{sec:adaptation}

\paragraph{Bootstrap adaptation.}
To isolate bootstrap adaptation, we compare adapting $\alpha_B$ from five with fixing it at five. Both conditions retain two actors, adapt $\alpha_E$ from five, and use the same actor and critic settings and coefficient learning rate. Table~\ref{tab:bootstrap-domains} shows higher averages with bootstrap adaptation in both domains, particularly on AntMaze (\BootstrapAntMazeFixed{} to \BootstrapAntMazeAdaptive{}). The overall mean increases from \BootstrapOverallFixed{} to \BootstrapOverallAdaptive{}. Gains are particularly large on AntMaze-large-play (20.0 to 48.0) and AntMaze-large-diverse (13.0 to 51.0). Appendix~\ref{app:mechanism} reports the per-task results and a fixed-$\alpha_B=1$ control.

\input{tables/bootstrap_domains}

\paragraph{Fixed learned endpoints.}
We next test whether fixed coefficients reproduce adaptation throughout training. For each environment, we retrain a two-actor agent with both coefficients fixed to the corresponding four-seed means of TD3+RAPO's final coefficients.

Table~\ref{tab:frozen-final} shows a higher overall mean with adaptation (83.0 versus 78.4), with higher means in both domains. The per-task results in Appendix~\ref{app:mechanism} show that the Locomotion difference is concentrated in Hopper-medium-expert, while four of six AntMaze tasks improve. Learned final coefficient values therefore do not reproduce the adaptive domain averages when fixed from the start.

\input{tables/frozen_final_domains}

\section{Related Work}
\label{sec:related}

\paragraph{Behavior regularization.}
Offline RL must balance exploiting learned values against selecting actions outside dataset coverage \citep{levine2020offline}. TD3+BC applies a quadratic BC penalty to the TD3 actor \citep{fujimoto2018td3,fujimoto2021minimalist}, while CQL regularizes value estimates toward pessimism \citep{kumar2020cql}. ReBRAC combines architectural changes with separate regularization coefficients for actor learning and critic targets \citep{tarasov2023rebrac}. These interventions act at different points in learning. RAPO retains the base value-loss form and adapts actor-loss coefficients, including the coefficient of the actor supplying bootstrap actions.

\paragraph{Adaptive constraints and meta-gradients.}
Weighted Policy Constraints favor desirable dataset actions \citep{peng2023wpc}, A2PR constructs an advantage-guided behavior reference \citep{liu2024a2pr}, and selective regularization varies constraint strength across states \citep{luo2025selective}. Meta-gradients learn update parameters by differentiating through optimization \citep{xu2018metagradient,franceschi2018bilevel}. ASPC uses this mechanism for policy-constraint scales, including IQL's inverse temperature \citep{jing2026aspc}. Its outer loss includes the squared change in mean Q, $(\E\Delta Q)^2$. RAPO instead assigns local improvement to execution and per-transition target RMS, $\sqrt{\E(\Delta y)^2}$, to bootstrapping. The latter prevents opposite-signed target changes from canceling. The distinction is the role-specific evaluation criterion, not the use of meta-gradients or value-change regularization itself.

\paragraph{Decoupling value learning and execution.}
Separating the policies involved in value learning and deployment has several precedents. One-step RL estimates the behavior policy's value and then performs a regularized improvement step without repeatedly evaluating the improved policy \citep{brandfonbrener2021onestep}. IQL and XQL learn values from dataset actions without requiring the extracted actor in their offline Bellman updates \citep{kostrikov2022iql,garg2023xql}. MCEP retains an explicit target policy and trains a separate, more mildly constrained evaluation policy \citep{xu2024mcep}. RAPO learns how strongly to regularize each actor by evaluating candidate updates according to the actor's role in execution or value learning.

\paragraph{Value learning and policy extraction.}
The extraction procedure determines how learned values are converted into executable actions. IDQL relates IQL's value objective to an implicit actor and uses critic-weighted samples from a diffusion behavior model for extraction \citep{hansenestruch2023idql}. Experiments that vary value learning and extraction independently show that the extraction objective can substantially change performance even with the same value function \citep{park2024bottleneck}. Other decoupled methods rank behavior-model proposals at inference time \citep{lin2026decoupling}. RAPO instead learns an explicit execution actor during training. Its IQL instantiation retains advantage-weighted regression and adapts its inverse temperature using a local critic-improvement score; it does not change the policy class or introduce inference-time search.

\paragraph{Local policy improvement.}
Performance-difference and trust-region analyses connect action-value improvement to return change under assumptions on the critic and state distribution \citep{schulman2015trpo}. RAPO uses a single candidate update to assess a coefficient, retaining the base actor objective. Its endpoint-gradient score is a local surrogate for this assessment; the conditional smoothness bound and the additional errors separating that score from return improvement are stated in Section~\ref{sec:execution} and Appendix~\ref{app:approximation}.

\section{Conclusion}
\label{sec:conclusion}

RAPO learns a policy-update coefficient by testing a candidate update against a criterion for the policy's use. In TD3+BC, this gives separate execution and bootstrap actors: the former is evaluated by a local critic-improvement surrogate, the latter by its effect on critic targets. In IQL, value updates already use dataset actions, so only the execution policy's extraction coefficient is adapted.

Both instantiations improve over their respective baselines in domain-average return on the fifteen D4RL tasks. Within the two-actor TD3 setting, the full execution score and adapting the bootstrap coefficient achieve higher domain averages than the tested controls; fixed learned endpoint coefficients also yield lower averages than continued adaptation. These findings motivate choosing coefficient-learning criteria according to how policy actions enter the algorithm. The execution score remains a local surrogate, and its relation to return improvement depends on the learned critic and the state distribution.

\subsection*{AI use statement}
Generative AI tools assisted with discovering relevant literature; drafting, editing, and polishing portions of the manuscript; and developing and checking mathematical arguments, including the local critic-improvement lemma and its proof in Appendix~\ref{app:localproof}. The authors reviewed the AI-assisted work and take responsibility for the final text, citations, proofs, technical claims, and results.

\bibliographystyle{plainnat}
\bibliography{references}

\clearpage
\appendix

\section{Analysis of the Execution Objective}
\label{app:approximation}

\subsection{Proof of Lemma~\ref{lem:local}}
\label{app:localproof}

\begin{proof}
Fix $s$ and write $f(a)=\bar Q_E(s,a)$. The fundamental theorem of calculus and the Lipschitz condition give
\begin{align*}
f(a_+)-f(a_0)-\nabla f(a_0)^\top\delta a
&=\int_0^1\!\bigl[\nabla f(a_0+t\delta a)-\nabla f(a_0)\bigr]^\top\delta a\,\mathrm dt\\
&\geq-\int_0^1 L_s t\|\delta a\|_2^2\,\mathrm dt
=-\frac{L_s}{2}\|\delta a\|_2^2.
\end{align*}
\end{proof}

The smoothness premise can fail at activation boundaries or, for IQL's minimum critic, where the minimizing Q-network switches.

\paragraph{Quadratic interpretation of the conservatism term.}
If $\bar Q_E(s,a)$ is quadratic in $a$ with symmetric Hessian $H_s$, then $g_+-g_0=H_s\delta a$ and
\begin{align*}
\bar Q_E(s,a_+)-\bar Q_E(s,a_0)
&=g_0^\top\delta a+\tfrac12\delta a^\top H_s\delta a,\\
\left|\tfrac12\delta a^\top H_s\delta a\right|
&\leq\tfrac12\|H_s\delta a\|_2\|\delta a\|_2
=\tfrac12\|g_+-g_0\|_2\|\delta a\|_2.
\end{align*}
Thus Equation~\eqref{eq:endpoint} subtracts an upper bound on the magnitude of the quadratic correction in this special case. For a general critic, endpoint variation need not control interior curvature; the score is not a certified lower bound. This curvature penalty also does not quantify uncertainty in the critic.

\subsection{Relation to Return Improvement}
\label{app:return}

For deterministic policies $\pi$ and $\pi^+$ with a common initial-state distribution, the performance-difference identity is \citep{schulman2015trpo}
\begin{equation}
\label{eq:pdl}
J(\pi^+)-J(\pi)
=\frac{1}{1-\gamma}\E_{s\sim d^{\pi^+}}
\left[Q^\pi(s,\pi^+(s))-Q^\pi(s,\pi(s))\right],
\end{equation}
where $d^{\pi^+}$ is normalized discounted state occupancy. To relate this identity to RAPO, set $\pi=\pi_E$, $\pi^+=\widetilde\pi_E$, let $\nu$ be the dataset state distribution, and define $\Delta_E(s)=\bar Q_E(s,\pi^+(s))-\bar Q_E(s,\pi(s))$. Assuming integrability, adding and subtracting $\Delta_E$ gives the exact decomposition
\begin{align}
\label{eq:decomposition}
(1-\gamma)\bigl[J(\pi^+)-J(\pi)\bigr]
&=\E_\nu\widehat B_\pi+\mathcal R_{\mathrm{local}}
+\mathcal R_{\mathrm{critic}}+\mathcal R_{\mathrm{state}},\\
\mathcal R_{\mathrm{local}}
&=\E_\nu\left[\Delta_E-\widehat B_\pi\right],\notag\\
\mathcal R_{\mathrm{critic}}
&=\E_{d^{\pi^+}}\left[
Q^\pi(s,\pi^+(s))-Q^\pi(s,\pi(s))-\Delta_E(s)\right],\notag\\
\mathcal R_{\mathrm{state}}
&=\E_{d^{\pi^+}}\Delta_E-\E_\nu\Delta_E.\notag
\end{align}
These residuals capture the endpoint approximation, critic error relative to $Q^\pi$ (including reference-policy mismatch), and the use of dataset states. RAPO does not bound them. Even if the learned critic is smooth, endpoint gradients need not control interior curvature; hence the execution score does not guarantee return improvement.

\section{Algorithm Instantiations and Implementation Details}
\label{app:settings}
\label{sec:instantiations}

\subsection{Differentiation and Normalization}
\label{app:differentiation}

The optimizer updates an unconstrained scalar $\rho_i$ for each active role, with the coefficient $c_i$ from Section~\ref{sec:roles}:
\[
\begin{array}{lll}
\text{TD3+RAPO:} & c_i=\alpha_i=\softplus(\rho_i), & i\in\{B,E\},\\
\text{IQL+RAPO:} & c_E=\beta_E=\exp(\rho_E), &
\end{array}
\qquad
\frac{\mathrm d\mathcal L_i}{\mathrm d\rho_i}
=\frac{\mathrm dc_i}{\mathrm d\rho_i}
\frac{\mathrm d\mathcal L_i}{\mathrm dc_i}.
\]
Each $\rho_i$ has a separate Adam state. The second factor follows Equation~\eqref{eq:meta} with the specified stop-gradient operations.

The candidate step leaves real actor parameters and optimizer state unchanged. Starting actor parameters, critic parameters, dataset-action references, and incoming optimizer states are constants for coefficient differentiation; new gradients and optimizer moments remain differentiable in $c_i$. Detached quantities are recomputed at each evaluation.

Using the critics from Sections~\ref{sec:execution} and~\ref{sec:bootstrap}, the detached value scales are defined as follows, with $\epsilon=10^{-6}$:
\begin{equation}
\label{eq:executionscale}
\bar S_E=\sg\!\left(\max\left\{\E_{\Bout}|\bar Q_E(s,\widetilde\pi_E(s))|,\epsilon\right\}\right),
\end{equation}
\begin{equation}
\label{eq:bootstrapnormalizer}
\bar S_B=\sg\!\left(\E_{\Bout}|\bar Q_{\min}(s,\widetilde\pi_B(s))|\right)+\epsilon.
\end{equation}
The bootstrap prefactor $\sg(\alpha_B)$ scales the coefficient gradient without adding a direct derivative path; dependence on $\alpha_B$ enters through the candidate policy. The bootstrap RMS includes terminal transitions as zeros and has no additive epsilon inside the square root; its value and chosen subgradient are zero when all inputs are zero.

\subsection{Training Settings}
\label{app:training}

The actor-critic implementations follow CORL \citep{tarasov2022corl}. For the TD3+BC, wPC, and ASPC comparisons, we use the same Q-network architecture as TD3+RAPO: three hidden layers of width 256 with ReLU activations and LayerNorm in each hidden layer.

Both actual and candidate TD3 actor updates use the loss scale in Equation~\eqref{eq:td3inner}; multiplying it by $\alpha_i$ changes finite updates and their meta-gradients. The training critic uses the target bootstrap actor with clipped Gaussian action noise ($\sigma=0.2$, $c=0.5$), whereas Equation~\eqref{eq:targetchange} compares unsmoothed online actors. Both methods mask the next-state value by $1-d$, where $d$ indicates termination.

IQL clips $\exp(\beta_E A)$ at 100. Gaussian actor means are used in the execution score and at evaluation.

Table~\ref{tab:methodsettings} summarizes RAPO's training settings. At initialization five, we select each method's coefficient learning rate from $\{3\times10^{-4},10^{-3},2\times10^{-3}\}$ by its four-seed mean in each environment. Table~\ref{tab:coefficient-rates} lists the selected rates.
\begin{table}[htbp]
\centering
\caption{Training settings. Intervals are measured in value-update steps; $\bar Q\leftarrow(1-\tau_Q)\bar Q+\tau_Q Q$ defines the target-update rate.}
\label{tab:methodsettings}
\begin{tabularx}{\textwidth}{l>{\centering\arraybackslash}X>{\centering\arraybackslash}X}
\toprule
Setting & TD3+RAPO & IQL+RAPO \\
\midrule
Actor hidden layers & $2\times256$ & $2\times256$ \\
Value hidden layers & Q: $3\times256$, LayerNorm & Q and V: $2\times256$, no LayerNorm \\
Network optimizer / initial rate & Adam / $3\times10^{-4}$ & Adam / $3\times10^{-4}$ \\
Actor learning-rate schedule & Constant & Cosine decay \\
Candidate optimizer & E: current-state Adam; B: SGD & E: current-state Adam \\
Actual actor optimizer / interval & Adam / 2 & Adam / 1 \\
Coefficient initialization & $\alpha_B=\alpha_E=\SelectedTDThreeAMOInit{}$ & $\beta_E=\SelectedIQLAMOInit{}$ \\
Coefficient update interval & 20 & 20, after 100,000 steps \\
Coefficient rate schedule & Exponential to $0.01$ of initial rate & Constant \\
Coefficient range & $\alpha_i>0$ & $\beta_E\in[0.05,100]$ \\
Target-update rate $\tau_Q$ & $0.005$ & $0.005$ \\
\bottomrule
\end{tabularx}
\end{table}

\input{tables/coefficient_rates}

Candidate execution steps use the actor's current Adam moments and learning rate. The bootstrap candidate uses SGD at the actor learning rate; the actual bootstrap update uses Adam. Network Adam uses moments $(0.9,0.999)$ and epsilon $10^{-8}$. TD3 coefficient Adam uses the same moments; IQL coefficient Adam uses $(0,0.999)$, with $\rho_E$ projected to $[\log0.05,\log100]$ after each update without resetting its moments.

For both instantiations, the actual execution update uses the coefficient cached before that step's meta-update. TD3 first updates its critics, then, on scheduled steps, updates $\rho_E$ and $\rho_B$, applies both real actor updates (using the new $\alpha_B$), and updates its target networks. The TD3 bootstrap target actor uses $\bar\theta_B\leftarrow(1-\tau_Q)\bar\theta_B+\tau_Q\theta_B$ with the same rate as its target critics. IQL caches the advantage and next-state value, updates V, Q, and target Q, performs the scheduled execution meta-update, and applies its real execution update.

Both instantiations use batch size 256, discount 0.99, normalized observations, and unchanged Locomotion rewards. TD3's smoothed target actions are clipped to $[-1,1]$.

IQL+RAPO uses Gaussian actors and an expectile of 0.7 on Locomotion and 0.9 on AntMaze. AntMaze rewards are shifted by $-1$. Actor means use tanh, and log standard deviations are state-independent and bounded in $[-20,2]$.

\paragraph{IQL baseline.}
The benchmark baseline uses the same actor family, target-update rate, expectiles, and reward handling, with fixed $\beta=3$ on Locomotion and $\beta=10$ on AntMaze.

\subsection{Evaluation and Configuration Selection}
\label{app:evaluation}
\label{app:selection}

The final-score and uncertainty protocol is defined in Section~\ref{sec:evaluationprotocol}. For domain uncertainty, environments are averaged within each seed before computing the sample standard deviation across seeds.

\subsection{Coefficient Learning-Rate Sensitivity}
\label{app:lr-sensitivity}

Table~\ref{tab:coefficient-lr-sensitivity} tests TD3+RAPO with each initial coefficient learning rate shared across environments; other settings follow Table~\ref{tab:methodsettings}. It exceeds TD3+BC's domain averages at all three rates, though task-level performance varies. At $3\times10^{-4}$, its Locomotion and AntMaze means are 87.4 and 64.2, compared with \ResultTDThreeBCLoco{} and \ResultTDThreeBCAnt{} for TD3+BC.

\input{tables/coefficient_lr_sensitivity}

\subsection{Outer-Objective and Coefficient Ablations}
\label{app:mechanism}

\paragraph{Bootstrap adaptation at main-selected settings.}
Table~\ref{tab:matched-bootstrap} extends the fixed-$\alpha_B=5$ comparison to fixed and adaptive values of one. All four conditions retain two actors, adapt $\alpha_E$ from five, and use TD3+RAPO's actor and critic settings and selected coefficient learning rate.

Figure~\ref{fig:fixed-bootstrap} extends the domain summary in Table~\ref{tab:bootstrap-domains}. Adapting $\alpha_B$ from one raises the Locomotion mean from \BootstrapLocomotionFixedOne{} to \BootstrapLocomotionAdaptiveOne{} and the AntMaze mean from \BootstrapAntMazeFixedOne{} to \BootstrapAntMazeAdaptiveOne{} relative to fixing it at one. Bootstrap adaptation therefore improves both domain averages at either initial value. Initialization still matters, especially on AntMaze, where adaptation from five reaches \BootstrapAntMazeAdaptive{} compared with \BootstrapAntMazeAdaptiveOne{} from one.

\begin{figure}[htbp]
\centering
\includegraphics[width=\textwidth]{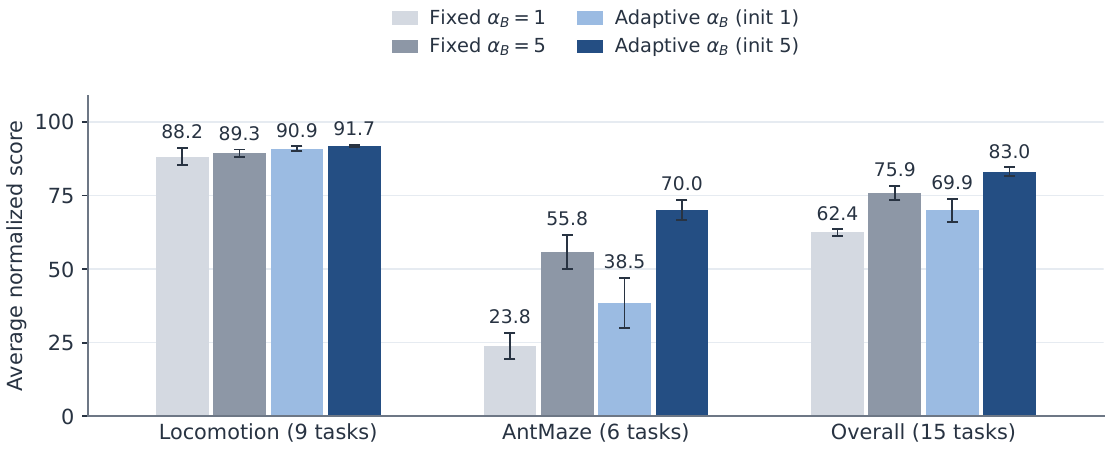}
\caption{Fixed and adaptive bootstrap regularization at TD3+RAPO's selected learning rates. Bars show domain or overall task averages; error bars show standard deviations across seed-wise task averages. All conditions adapt $\alpha_E$ from five; adaptive $\alpha_B$ starts at one or five. Per-task results appear in Table~\ref{tab:matched-bootstrap}.}
\label{fig:fixed-bootstrap}
\end{figure}

\input{tables/matched_bootstrap}

\clearpage
\paragraph{Fixed learned endpoints.}
For each environment, we use the mean final coefficients of the four TD3+RAPO seeds as a fixed pair, rounded to six decimal places for training. We train new actors and critics with coefficient updates disabled. Table~\ref{tab:frozen-final-tasks} displays the constants and per-task results; the adaptive column uses the main benchmark.

\input{tables/frozen_final_tasks}

\end{document}

%% file: tables/results_macros.tex
\newcommand{\ResultTDThreeBCLoco}{80.8}
\newcommand{\ResultIQLLoco}{78.6}

\newcommand{\ResultTDThreeAMOLoco}{91.7}
\newcommand{\ResultIQLAMOLoco}{81.2}
\newcommand{\ResultTDThreeBCAnt}{31.7}
\newcommand{\ResultIQLAnt}{55.4}

\newcommand{\ResultTDThreeAMOAnt}{70.0}
\newcommand{\ResultIQLAMOAnt}{62.1}

\newcommand{\SelectedTDThreeAMOInit}{5}

\newcommand{\SelectedIQLAMOInit}{5}

%% file: tables/bootstrap_macros.tex
\newcommand{\BootstrapLocomotionFixedOne}{88.2}

\newcommand{\BootstrapAntMazeFixedOne}{23.8}
\newcommand{\BootstrapAntMazeFixed}{55.8}
\newcommand{\BootstrapAntMazeAdaptive}{70.0}

\newcommand{\BootstrapOverallFixed}{75.9}
\newcommand{\BootstrapOverallAdaptive}{83.0}

%% file: tables/adaptive_bootstrap_macros.tex
\newcommand{\BootstrapLocomotionAdaptiveOne}{90.9}
\newcommand{\BootstrapAntMazeAdaptiveOne}{38.5}

%% file: tables/execution_score_macros.tex
\newcommand{\ExecutionDirectQOverall}{78.8}
\newcommand{\ExecutionFirstOrderOverall}{79.4}
\newcommand{\ExecutionFullOverall}{83.0}

%% file: tables/benchmark_results.tex
\begin{table}[H]
\centering
\caption{D4RL performance at one million training steps, reported as mean $\pm$ sample standard deviation across four training seeds. Initial coefficients are fixed at five for both RAPO methods. Hyperparameters are reported in Appendix~\ref{app:training}. Average rows give equal weight to environments; the final row averages all fifteen tasks. Bold marks the highest mean in each row, including ties.}
\label{tab:benchmark}
\begingroup
\setlength{\tabcolsep}{2pt}
\renewcommand{\arraystretch}{1.12}
\resizebox{\textwidth}{!}{%
\begin{tabular}{l*{7}{c}}
\toprule
Dataset & TD3+BC & IQL & wPC & A2PR & ASPC & TD3+RAPO & IQL+RAPO \\
\midrule
HalfCheetah-m & $50.0\!\pm\!0.7$ & $48.4\!\pm\!0.2$ & $54.7\!\pm\!0.4$ & $55.9\!\pm\!0.6$ & $57.8\!\pm\!0.7$ & $\boldsymbol{58.2\!\pm\!0.5}$ & $48.6\!\pm\!0.2$ \\
HalfCheetah-mr & $46.2\!\pm\!0.5$ & $44.2\!\pm\!0.4$ & $48.2\!\pm\!0.4$ & $48.0\!\pm\!1.1$ & $\boldsymbol{50.1\!\pm\!1.3}$ & $49.7\!\pm\!0.7$ & $44.6\!\pm\!0.5$ \\
HalfCheetah-me & $99.5\!\pm\!1.0$ & $89.5\!\pm\!3.9$ & $100.0\!\pm\!0.5$ & $98.9\!\pm\!1.7$ & $92.2\!\pm\!12.8$ & $\boldsymbol{100.6\!\pm\!3.0}$ & $92.8\!\pm\!0.6$ \\
Hopper-m & $61.5\!\pm\!2.7$ & $63.0\!\pm\!2.5$ & $83.9\!\pm\!6.8$ & $83.4\!\pm\!2.5$ & $81.9\!\pm\!5.2$ & $\boldsymbol{101.8\!\pm\!0.2}$ & $64.4\!\pm\!5.8$ \\
Hopper-mr & $78.8\!\pm\!34.2$ & $87.6\!\pm\!8.4$ & $101.2\!\pm\!0.6$ & $100.1\!\pm\!1.1$ & $99.6\!\pm\!4.0$ & $\boldsymbol{101.5\!\pm\!0.6}$ & $94.1\!\pm\!15.4$ \\
Hopper-me & $111.4\!\pm\!2.0$ & $99.8\!\pm\!24.0$ & $103.5\!\pm\!7.4$ & $\boldsymbol{112.6\!\pm\!0.5}$ & $109.2\!\pm\!2.7$ & $111.0\!\pm\!1.3$ & $108.3\!\pm\!5.1$ \\
Walker2d-m & $83.7\!\pm\!2.8$ & $81.4\!\pm\!3.5$ & $88.9\!\pm\!1.4$ & $89.3\!\pm\!1.3$ & $89.5\!\pm\!15.4$ & $\boldsymbol{93.8\!\pm\!1.4}$ & $83.2\!\pm\!2.2$ \\
Walker2d-mr & $86.7\!\pm\!5.6$ & $81.5\!\pm\!4.5$ & $93.5\!\pm\!3.5$ & $91.5\!\pm\!5.1$ & $90.6\!\pm\!9.4$ & $\boldsymbol{98.0\!\pm\!1.4}$ & $82.4\!\pm\!6.5$ \\
Walker2d-me & $109.7\!\pm\!0.7$ & $111.8\!\pm\!1.2$ & $110.3\!\pm\!0.5$ & $112.0\!\pm\!0.3$ & $110.5\!\pm\!0.2$ & $111.1\!\pm\!1.2$ & $\boldsymbol{112.5\!\pm\!0.8}$ \\
\midrule
Locomotion average & $80.8$ & $78.6$ & $87.1$ & $88.0$ & $86.8$ & $\boldsymbol{91.7}$ & $81.2$ \\
\midrule
Ant-umaze & $90.0\!\pm\!8.2$ & $76.0\!\pm\!14.0$ & $\boldsymbol{92.5\!\pm\!9.6}$ & $\boldsymbol{92.5\!\pm\!5.0}$ & $88.0\!\pm\!3.7$ & $\boldsymbol{92.5\!\pm\!6.4}$ & $77.5\!\pm\!10.8$ \\
Ant-umaze-diverse & $87.5\!\pm\!5.0$ & $58.0\!\pm\!5.9$ & $80.0\!\pm\!8.2$ & $55.0\!\pm\!41.2$ & $82.5\!\pm\!4.7$ & $\boldsymbol{88.0\!\pm\!6.9}$ & $63.0\!\pm\!7.0$ \\
Ant-medium-play & $2.5\!\pm\!5.0$ & $63.5\!\pm\!6.2$ & $70.0\!\pm\!14.1$ & $45.0\!\pm\!30.0$ & $76.0\!\pm\!12.4$ & $\boldsymbol{81.5\!\pm\!3.0}$ & $73.0\!\pm\!6.8$ \\
Ant-medium-diverse & $10.0\!\pm\!8.2$ & $66.0\!\pm\!13.1$ & $60.0\!\pm\!14.1$ & $47.5\!\pm\!9.6$ & $49.5\!\pm\!22.3$ & $59.0\!\pm\!17.2$ & $\boldsymbol{73.0\!\pm\!11.8}$ \\
Ant-large-play & $0.0\!\pm\!0.0$ & $35.0\!\pm\!10.4$ & $32.5\!\pm\!5.0$ & $15.0\!\pm\!12.9$ & $\boldsymbol{53.5\!\pm\!19.7}$ & $48.0\!\pm\!6.9$ & $45.5\!\pm\!5.7$ \\
Ant-large-diverse & $0.0\!\pm\!0.0$ & $34.0\!\pm\!11.0$ & $45.0\!\pm\!17.3$ & $15.0\!\pm\!17.3$ & $32.5\!\pm\!20.9$ & $\boldsymbol{51.0\!\pm\!5.3}$ & $40.5\!\pm\!4.4$ \\
\midrule
AntMaze average & $31.7$ & $55.4$ & $63.3$ & $45.0$ & $63.7$ & $\boldsymbol{70.0}$ & $62.1$ \\
\midrule
Overall average & $61.2$ & $69.3$ & $77.6$ & $70.8$ & $77.6$ & $\boldsymbol{83.0}$ & $73.5$ \\
\bottomrule
\end{tabular}%
}
\endgroup
\end{table}

%% file: tables/execution_score_domains.tex
\begin{table}[htbp]
\centering
\caption{Execution-score ablation at the selected coefficient learning rates. Q uses target-critic value change; Linear uses the first-order gain. Entries are equal-weight task means within each domain and overall. Bold marks the highest mean.}
\label{tab:execution-scores}
\begin{tabular}{lccc}
\toprule
Domain & Q & Linear & $\widehat B_\pi$ (Ours) \\
\midrule
Locomotion & $87.0$ & $88.7$ & $\mathbf{91.7}$ \\
AntMaze & $66.6$ & $65.5$ & $\mathbf{70.0}$ \\
Overall & $78.8$ & $79.4$ & $\mathbf{83.0}$ \\
\bottomrule
\end{tabular}
\end{table}

%% file: tables/bootstrap_domains.tex
\begin{table}[htbp]
\centering
\caption{Bootstrap-coefficient adaptation with $\alpha_E$ adapted in both conditions. The control fixes $\alpha_B=5$; TD3+RAPO adapts it from five. Entries are equal-weight task means within each domain and overall. Per-task results appear in Table~\ref{tab:matched-bootstrap}.}
\label{tab:bootstrap-domains}
\begin{tabular}{lcc}
\toprule
Domain & Fixed $\alpha_B=5$ & TD3+RAPO \\
\midrule
Locomotion & $89.3$ & $91.7$ \\
AntMaze & $55.8$ & $70.0$ \\
Overall & $75.9$ & $83.0$ \\
\bottomrule
\end{tabular}
\end{table}

%% file: tables/frozen_final_domains.tex
\begin{table}[htbp]
\centering
\caption{Fixing learned endpoint coefficients from the start versus adapting from five. Both conditions retain separate TD3 actors. Scores are equal-weight task means; fixed constants and per-task results appear in Table~\ref{tab:frozen-final-tasks}.}
\label{tab:frozen-final}
\begin{tabular}{lcc}
\toprule
Domain & Fixed endpoints & TD3+RAPO \\
\midrule
Locomotion & $88.7$ & $91.7$ \\
AntMaze & $62.9$ & $70.0$ \\
Overall & $78.4$ & $83.0$ \\
\bottomrule
\end{tabular}
\end{table}

%% file: tables/coefficient_rates.tex
\begin{table}[htbp]
\centering
\caption{Main-selected coefficient learning rates with initialization five. TD3 uses the listed rate for both coefficients; its fixed-bootstrap control reuses the same rate for $\alpha_E$. Fixed-endpoint TD3 has no active coefficient update.}
\label{tab:coefficient-rates}
\begin{tabular}{lcc}
\toprule
Environment & TD3+RAPO & IQL+RAPO \\
\midrule
HalfCheetah-medium & $2\times10^{-3}$ & $3\times10^{-4}$ \\
HalfCheetah-medium-replay & $2\times10^{-3}$ & $10^{-3}$ \\
HalfCheetah-medium-expert & $10^{-3}$ & $10^{-3}$ \\
Hopper-medium & $2\times10^{-3}$ & $3\times10^{-4}$ \\
Hopper-medium-replay & $2\times10^{-3}$ & $3\times10^{-4}$ \\
Hopper-medium-expert & $3\times10^{-4}$ & $10^{-3}$ \\
Walker2d-medium & $10^{-3}$ & $3\times10^{-4}$ \\
Walker2d-medium-replay & $2\times10^{-3}$ & $10^{-3}$ \\
Walker2d-medium-expert & $2\times10^{-3}$ & $10^{-3}$ \\
AntMaze-umaze & $3\times10^{-4}$ & $10^{-3}$ \\
AntMaze-umaze-diverse & $3\times10^{-4}$ & $3\times10^{-4}$ \\
AntMaze-medium-play & $10^{-3}$ & $2\times10^{-3}$ \\
AntMaze-medium-diverse & $3\times10^{-4}$ & $3\times10^{-4}$ \\
AntMaze-large-play & $10^{-3}$ & $2\times10^{-3}$ \\
AntMaze-large-diverse & $10^{-3}$ & $10^{-3}$ \\
\bottomrule
\end{tabular}
\end{table}

%% file: tables/coefficient_lr_sensitivity.tex
\begin{table}[htbp]
\centering
\setlength{\tabcolsep}{2pt}
\caption{TD3+RAPO sensitivity to the initial coefficient learning rate, with coefficient initialization five. Average rows weight environments equally.}
\label{tab:coefficient-lr-sensitivity}
\input{tables/coefficient_lr_sensitivity_body}
\end{table}

%% file: tables/coefficient_lr_sensitivity_body.tex
\begin{tabular*}{\textwidth}{@{\extracolsep{\fill}}lccc@{}}
\toprule
Environment & $3\!\times\!10^{-4}$ & $10^{-3}$ & $2\!\times\!10^{-3}$ \\
\midrule
HalfCheetah-medium & $54.3 \pm 0.6$ & $56.4 \pm 0.7$ & $58.2 \pm 0.5$ \\
HalfCheetah-medium-replay & $47.9 \pm 0.5$ & $49.3 \pm 0.3$ & $49.7 \pm 0.7$ \\
HalfCheetah-medium-expert & $98.1 \pm 2.8$ & $100.6 \pm 3.0$ & $98.0 \pm 3.9$ \\
Hopper-medium & $83.9 \pm 11.2$ & $96.4 \pm 5.7$ & $101.8 \pm 0.2$ \\
Hopper-medium-replay & $96.7 \pm 6.4$ & $99.8 \pm 2.6$ & $101.5 \pm 0.6$ \\
Hopper-medium-expert & $111.0 \pm 1.3$ & $79.1 \pm 20.3$ & $45.4 \pm 5.5$ \\
Walker2d-medium & $89.6 \pm 1.7$ & $93.8 \pm 1.4$ & $88.7 \pm 12.7$ \\
Walker2d-medium-replay & $95.0 \pm 2.5$ & $94.0 \pm 4.5$ & $98.0 \pm 1.4$ \\
Walker2d-medium-expert & $110.0 \pm 0.5$ & $110.8 \pm 0.7$ & $111.0 \pm 1.2$ \\
\midrule
Locomotion Avg. & 87.4 & 86.7 & 83.6 \\
\midrule
AntMaze-umaze & $92.5 \pm 6.4$ & $88.0 \pm 4.3$ & $85.0 \pm 7.4$ \\
AntMaze-umaze-diverse & $88.0 \pm 6.9$ & $35.5 \pm 29.9$ & $77.0 \pm 12.7$ \\
AntMaze-medium-play & $72.5 \pm 14.9$ & $81.5 \pm 3.0$ & $71.0 \pm 9.5$ \\
AntMaze-medium-diverse & $59.0 \pm 17.2$ & $43.0 \pm 20.0$ & $26.5 \pm 19.8$ \\
AntMaze-large-play & $29.5 \pm 13.4$ & $48.0 \pm 6.9$ & $46.5 \pm 12.6$ \\
AntMaze-large-diverse & $43.5 \pm 5.3$ & $51.0 \pm 5.3$ & $30.0 \pm 14.6$ \\
\midrule
AntMaze Avg. & 64.2 & 57.8 & 56.0 \\
\midrule
Overall Avg. & 78.1 & 75.1 & 72.6 \\
\bottomrule
\end{tabular*}

%% file: tables/matched_bootstrap.tex
\begin{table}[htbp]
\centering
\setlength{\tabcolsep}{4pt}
\caption{Fixed and adaptive bootstrap regularization at the selected learning rates. All conditions retain two actors and adapt $\alpha_E$ from five. Under Fixed, column numbers are constant $\alpha_B$ values; under Adaptive, they are initial values.}
\label{tab:matched-bootstrap}
\begin{tabular}{lcccc}
\toprule
Environment & \multicolumn{2}{c}{Fixed $\alpha_B$} & \multicolumn{2}{c}{Adaptive $\alpha_B$} \\
\cmidrule(lr){2-3}\cmidrule(lr){4-5}
 & 1 & 5 & 1 & 5 \\
\midrule
HalfCheetah-medium & $56.6\pm0.9$ & $59.1\pm0.6$ & $58.1\pm1.0$ & $58.2\pm0.5$ \\
HalfCheetah-medium-replay & $48.7\pm2.4$ & $50.0\pm0.6$ & $49.1\pm0.5$ & $49.7\pm0.7$ \\
HalfCheetah-medium-expert & $103.1\pm3.6$ & $98.8\pm3.7$ & $95.3\pm5.7$ & $100.6\pm3.0$ \\
Hopper-medium & $88.5\pm24.2$ & $101.7\pm0.5$ & $101.9\pm0.7$ & $101.8\pm0.2$ \\
Hopper-medium-replay & $98.7\pm1.3$ & $101.3\pm1.2$ & $100.5\pm1.5$ & $101.5\pm0.6$ \\
Hopper-medium-expert & $111.0\pm2.3$ & $99.1\pm10.1$ & $112.5\pm0.2$ & $111.0\pm1.3$ \\
Walker2d-medium & $85.0\pm0.2$ & $90.3\pm3.1$ & $93.2\pm1.8$ & $93.8\pm1.4$ \\
Walker2d-medium-replay & $91.9\pm1.8$ & $92.6\pm9.7$ & $96.5\pm3.6$ & $98.0\pm1.4$ \\
Walker2d-medium-expert & $110.2\pm0.8$ & $110.6\pm0.8$ & $110.9\pm0.8$ & $111.1\pm1.2$ \\
\midrule
Locomotion average & $88.2$ & $89.3$ & $90.9$ & $91.7$ \\
\midrule
AntMaze-umaze & $84.5\pm4.4$ & $90.0\pm4.3$ & $90.0\pm1.6$ & $92.5\pm6.4$ \\
AntMaze-umaze-diverse & $58.5\pm26.3$ & $88.5\pm12.0$ & $57.5\pm39.5$ & $88.0\pm6.9$ \\
AntMaze-medium-play & $0.0\pm0.0$ & $62.0\pm5.4$ & $46.5\pm17.0$ & $81.5\pm3.0$ \\
AntMaze-medium-diverse & $0.0\pm0.0$ & $61.5\pm16.3$ & $6.5\pm6.0$ & $59.0\pm17.2$ \\
AntMaze-large-play & $0.0\pm0.0$ & $20.0\pm4.3$ & $24.5\pm15.3$ & $48.0\pm6.9$ \\
AntMaze-large-diverse & $0.0\pm0.0$ & $13.0\pm8.9$ & $6.0\pm9.5$ & $51.0\pm5.3$ \\
\midrule
AntMaze average & $23.8$ & $55.8$ & $38.5$ & $70.0$ \\
Overall average & $62.4$ & $75.9$ & $69.9$ & $83.0$ \\
\bottomrule
\end{tabular}
\end{table}

%% file: tables/frozen_final_tasks.tex
\begin{table}[H]
\centering
\setlength{\tabcolsep}{4pt}
\caption{Endpoint-fixed TD3 control. Each coefficient pair averages the final values of the main runs and remains fixed throughout the control runs. Constants are displayed to one decimal place; training uses six decimal places. TD3+RAPO uses the main benchmark runs.}
\label{tab:frozen-final-tasks}
\begin{tabular}{lcccc}
\toprule
Environment & Fixed $\alpha_E$ & Fixed $\alpha_B$ & Fixed endpoints & TD3+RAPO \\
\midrule
HalfCheetah-medium & 19.4 & 13.4 & $58.3\pm0.4$ & $58.2\pm0.5$ \\
HalfCheetah-medium-replay & 18.8 & 17.4 & $49.9\pm0.9$ & $49.7\pm0.7$ \\
HalfCheetah-medium-expert & 12.6 & 10.2 & $98.8\pm5.4$ & $100.6\pm3.0$ \\
Hopper-medium & 19.6 & 17.0 & $101.9\pm0.4$ & $101.8\pm0.2$ \\
Hopper-medium-replay & 23.3 & 16.7 & $101.0\pm1.1$ & $101.5\pm0.6$ \\
Hopper-medium-expert & 7.1 & 7.1 & $83.8\pm20.3$ & $111.0\pm1.3$ \\
Walker2d-medium & 11.8 & 9.8 & $95.5\pm6.8$ & $93.8\pm1.4$ \\
Walker2d-medium-replay & 20.7 & 18.2 & $98.2\pm2.6$ & $98.0\pm1.4$ \\
Walker2d-medium-expert & 19.3 & 12.8 & $111.1\pm0.5$ & $111.1\pm1.2$ \\
AntMaze-umaze & 7.8 & 6.7 & $95.5\pm3.4$ & $92.5\pm6.4$ \\
AntMaze-umaze-diverse & 7.5 & 6.3 & $61.5\pm35.9$ & $88.0\pm6.9$ \\
AntMaze-medium-play & 13.7 & 9.0 & $75.5\pm9.1$ & $81.5\pm3.0$ \\
AntMaze-medium-diverse & 7.7 & 6.0 & $49.0\pm17.3$ & $59.0\pm17.2$ \\
AntMaze-large-play & 13.5 & 9.2 & $53.0\pm6.8$ & $48.0\pm6.9$ \\
AntMaze-large-diverse & 13.7 & 9.3 & $43.0\pm13.2$ & $51.0\pm5.3$ \\
\bottomrule
\end{tabular}
\end{table}